\documentclass[runningheads]{llncs}
\usepackage[english]{babel}
\usepackage{comment}

\usepackage{booktabs} % For pretty tables
\usepackage{subcaption} % For sub-figures
\usepackage{graphicx}
\usepackage{amsfonts}
\usepackage{amsmath}
\usepackage[left]{lineno} % 'left' aligns numbers on the left
\usepackage{xcolor}
\usepackage{algorithm}
\usepackage{algorithmic}
\usepackage{multirow}
\usepackage[normalem]{ulem}
\usepackage{amssymb}

\usepackage{tikz}
\usetikzlibrary{arrows.meta,positioning,fit,backgrounds,calc,shapes.geometric}

\newcommand{\vect}[1]{\mathbf{#1}}

\begin{document}
\title{Instance-Adaptive Parametrization for Amortized Variational Inference}
\titlerunning{Instance-Adaptive Parametrization for Amortized Variational Inference}
\author{
Andrea Pollastro\inst{1} \and
Andrea Apicella\inst{2} \and
Francesco Isgrò$^{\dag}$\inst{1} \and
Roberto Prevete$^{\dag}$\inst{1}
}
\authorrunning{A. Pollastro, A. Apicella, F. Isgrò, and R. Prevete}
\institute{
Department of Electrical Engineering and Information Technology, University of Naples Federico II, 80125, Naples, Italy
\and
Department of Information Engineering, Electrical Engineering, and Applied Mathematics, University of Salerno, 84084, Fisciano (Salerno), Italy
}
\maketitle

\begingroup
\renewcommand\thefootnote{\dag}
\footnotetext{These authors share senior authorship.}
\endgroup

% typeset the header of the contribution
%
% \linenumbers % Activate line numbering

\begin{abstract}
Variational autoencoders (VAEs) rely on amortized variational inference to enable efficient posterior approximation, but this efficiency comes at the cost of a shared parametrization, giving rise to the amortization gap.
We propose the \emph{instance-adaptive variational autoencoder} (IA-VAE), an amortized inference framework in which a hypernetwork generates input-dependent modulations of a shared encoder. This enables input-specific adaptation of the inference model while preserving the efficiency of a single forward pass.
From a theoretical perspective, we show that the variational family induced by IA-VAE contains that of standard amortized inference, implying that IA-VAE cannot yield a worse optimal ELBO.
By leveraging instance-specific parameter modulations, the proposed approach can achieve performance comparable to standard encoders with substantially fewer parameters, indicating a more efficient use of model capacity.
Experiments on synthetic data, where the true posterior is known, show that IA-VAE yields more accurate posterior approximations and reduces the amortization gap.
Similarly, on standard image benchmarks, IA-VAE consistently improves held-out ELBO over baseline VAEs, with statistically significant gains across multiple runs.
These results suggest that increasing the flexibility of the inference parametrization through instance-adaptive modulation is an effective strategy for mitigating amortization-induced suboptimality in deep generative models.

\keywords{Variational Autoencoders \and Hypernetworks \and Amortized Variational Inference \and Amortization Gap}
\end{abstract}
\section{Introduction}
\label{sec:introduction}
% \roberto{In this work, we propose a \textit{hypernetwork}-based approach that goes beyond standard \textit{amortized inference} by conditioning a shared inference model on each input, enabling more flexible and expressive predictions across diverse datapoints. To motivate our approach, we briefly recall some standard formulations of \textit{latent} variable models.} 
Amortized variational inference~\cite{graves2011practical,kingma2013auto} enables scalable posterior approximation through shared inference networks and is a key component of modern deep generative modeling. 
However, this efficiency comes at a cost: a single global mapping constrains the ability to recover input-specific optimal variational parameters, giving rise to the amortization gap.
We propose a hypernetwork-based approach that relaxes this limitation via data-dependent parameter modulation of a shared inference model.
% We first briefly recall the latent variable model setting in which amortized inference is typically applied.

Many probabilistic models introduce latent variables to capture unobserved structure underlying observed data~\cite{blei2017variational}. In such settings, exact posterior inference is typically intractable, making approximate inference necessary.

Variational inference provides a general framework for approximating intractable posteriors by optimizing over a parameterized family of distributions~\cite{blei2017variational}, casting posterior approximation as an optimization problem and enabling scalable and efficient inference in complex models.

In recent years, variational inference has been increasingly combined with function approximators such as artificial neural networks (ANNs), leading to the development of amortized variational inference~\cite{graves2011practical,kingma2013auto}. 
In this context, a shared inference network maps each observation to the parameters of its approximate posterior, in contrast to non-amortized approaches that require instance-specific optimization. 
This enables efficient inference across large datasets, as posterior parameters are obtained via a single forward pass.
Among these methods, variational autoencoders (VAEs)~\cite{kingma2013auto,rezende2014stochastic} constitute a prominent instance of amortized variational inference in latent variable models. 
A graphical representation is shown in Figure~\ref{fig:amortized_vs_nonamortized}.

\begin{figure}[t]
\centering
\begin{tikzpicture}[
    font=\small,
    >=Latex,
    box/.style={draw, rounded corners, minimum width=2cm, minimum height=0.5cm, align=center, font=\scriptsize},
    smallbox/.style={draw, rounded corners, minimum width=1.7cm, minimum height=0.8cm, align=center, font=\scriptsize},
    % data/.style={draw, circle, minimum size=0.75cm, align=center},
    data/.style={draw, rounded corners, minimum size=0.75cm, minimum height=0.8cm, align=center, font=\scriptsize},
    title/.style={font=\bfseries},
    note/.style={font=\scriptsize, align=center}
]

% =========================
% LEFT PANEL: NON-AMORTIZED
% =========================
\node[title] at (-4.1,3.8) {Non-amortized inference};

% Inputs
\node[data, fill=green!5] (x1l) at (-6.5,2.5) {Data\\1};
\node[data, fill=green!5] (x2l) at (-6.5,1.0) {Data\\2};
\node[data, fill=green!5] (x3l) at (-6.5,-0.5) {Data\\3};

% Optimization blocks
% \node[box, fill=gray!5] (opt1) at (-4.6,2.5) {Optimization};
% \node[box, fill=gray!5] (opt2) at (-4.6,1.0) {Optimization};
% \node[box, fill=gray!5] (opt3) at (-4.6,-0.5) {Optimization};

% Outputs
\node[smallbox, fill=orange!10] (q1l) at (-2.3,2.5) {Variational\\Parameters};
\node[smallbox, fill=orange!10] (q2l) at (-2.3,1.0) {Variational\\Parameters};
\node[smallbox, fill=orange!10] (q3l) at (-2.3,-0.5) {Variational\\Parameters};

% Arrows
% \draw[->] (x1l) -- (opt1);
% \draw[->] (x2l) -- (opt2);
% \draw[->] (x3l) -- (opt3);

% \draw[->] (opt1) -- (q1l);
% \draw[->] (opt2) -- (q2l);
% \draw[->] (opt3) -- (q3l);

\draw[->] (x1l) -- (q1l) node[midway, above, font=\scriptsize] {Optimization};
\draw[->] (x2l) -- (q2l) node[midway, above, font=\scriptsize] {Optimization};
\draw[->] (x3l) -- (q3l) node[midway, above, font=\scriptsize] {Optimization};

\node[note] at (-4.2,-1.8) {Separate optimization\\for each datapoint};

% ======================
% RIGHT PANEL: AMORTIZED
% ======================
\node[title] at (2.2,3.8) {Amortized inference};

% Inputs
\node[data, fill=green!5] (x1r) at (.2,2.5) {Data\\1};
\node[data, fill=green!5] (x2r) at (.2,1.0) {Data\\2};
\node[data, fill=green!5] (x3r) at (.2,-0.5) {Data\\3};

% Shared network
\node[box, minimum width=1.4cm, minimum height=3.8cm, fill=blue!5] (enc) at (1.8,1.0) {Shared\\Inference\\Model};

% Outputs
\node[smallbox, fill=orange!10] (q1r) at (3.8,2.5) {Variational\\Parameters};
\node[smallbox, fill=orange!10] (q2r) at (3.8,1.0) {Variational\\Parameters};
\node[smallbox, fill=orange!10] (q3r) at (3.8,-0.5) {Variational\\Parameters};

% Arrows
\draw[->] (x1r) -- (enc.west|-x1r);
\draw[->] (x2r) -- (enc.west);
\draw[->] (x3r) -- (enc.west|-x3r);

\draw[->] (enc.east|-x1r) -- (q1r);
\draw[->] (enc.east) -- (q2r);
\draw[->] (enc.east|-x3r) -- (q3r);

\node[note] at (2.2,-1.8) {Single shared model predicts\\variational parameters directly};

\end{tikzpicture}
\caption{Comparison between non-amortized and amortized variational inference. 
In non-amortized inference (left), the variational parameters are optimized independently for each datapoint. 
In amortized inference (right), a shared inference model directly maps each observation to the parameters of its approximate posterior.}
\label{fig:amortized_vs_nonamortized}
\end{figure}
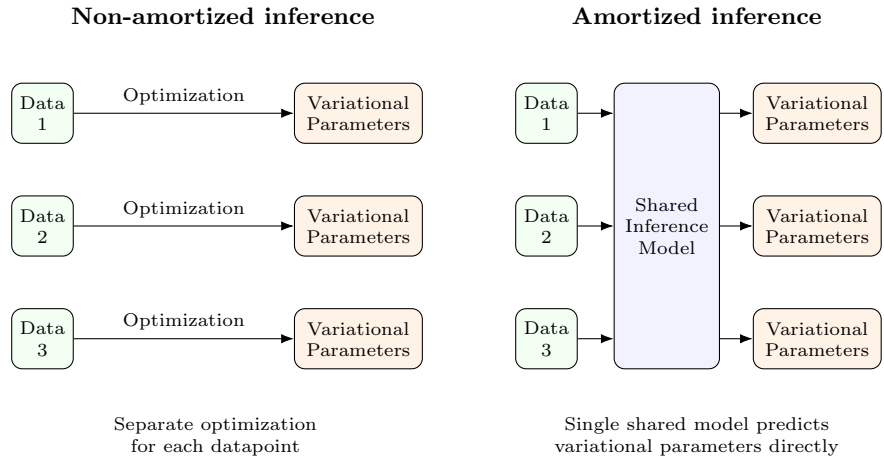

By coupling a parameterized inference network with a flexible generative model, VAEs enable scalable learning and approximate posterior inference in high-dimensional settings. 
%Owing to their computational efficiency and representational flexibility, VAEs have become a central framework in deep generative modeling and continue to be widely adopted across a broad range of application domains, including anomaly detection~\cite{pollastro2025sincvae,pollastro2023semi,kapsecker2025disentangled}, medical data analysis~\cite{sadria2025scvaeder,yue2026precision,ruiperez2026reducing,al2026multi}, synthetic data generation and augmentation~\cite{zhao2025ai,shang2026data,fathnejat2026augmentation}, and computational physics~\cite{nerin2025parameter,ren2026expanding}.
Owing to their computational efficiency and representational flexibility, VAEs have become a central framework in deep generative modeling and are widely applied across a broad range of domains \cite{pollastro2025sincvae,pollastro2023semi,kapsecker2025disentangled,sadria2025scvaeder,yue2026precision,ruiperez2026reducing,al2026multi,zhao2025ai,shang2026data,fathnejat2026augmentation,nerin2025parameter,ren2026expanding}.

%However, while amortization significantly improves computational scalability, it also introduces additional constraints on the variational approximation. Since the variational parameters are constrained to the image of a global inference function, the optimal instance-specific variational solution may not be obtainable. 
% \apicella{non mi è chiaro il precedente periodo} \pollastro{sopra diciamo che fai normalmente ottimizzazione per dato; qui se ammortizzi con una funzione ne guadagni di velocita', ma quello che ammortizzi e' vincolato dal codominio della funzione che usi, e quindi ne perdi. Quale parte non ti e' chiara?}
%More generally, the inability of inference networks to recover the input-specific optimal variational parameters introduces a structural source of suboptimality that can negatively affect both posterior accuracy and generative performance.
%This systematic discrepancy between amortized and instance-specific optimal variational solutions is commonly referred to as the \emph{amortization gap}~\cite{cremer2018inference,marino2018iterative,ganguly2023amortized}.
%A distinct source of suboptimality, known as the \emph{approximation gap}~\cite{cremer2018inference}, arises from the mismatch between the true posterior and the chosen variational family; however, addressing this aspect falls outside the scope of the present work. 

However, while amortization significantly improves computational scalability, it also introduces additional constraints on the variational approximation. Since the variational parameters are restricted to the image of a global inference function, the optimal instance-specific solution may not be attainable.
This limitation is commonly referred to as the \emph{amortization gap}~\cite{cremer2018inference,marino2018iterative,ganguly2023amortized}, and can negatively affect both posterior accuracy and generative performance. A distinct source of suboptimality, known as the \emph{approximation gap}~\cite{cremer2018inference}, arises from the mismatch between the true posterior and the chosen variational family, but is not the focus of this work.
A graphical representation of the inference error in VAEs is shown in Figure~\ref{fig:amortization_gap_diagram}.
% \apicella{la figura 2 si riferisce solo al primo punto o ad entrambi i punti? se solo al primo, la richiamerei prima. Inoltre} SI RIFERISCE AD ENTRAMBI, E' EVIDENTE NELLA FIGURE DOVE CI SONO SIA AMORTIZ CHE APPROX
\begin{figure}[t]
    \centering
    \begin{tikzpicture}[
        node distance=1.2cm,
        font=\sffamily\small,
        >=Stealth
    ]
        \draw[fill=gray!3, draw=gray!60, dashed, rounded corners=12pt] (-1, -0.3) rectangle (9.5, 5.9);
        \node[anchor=north east, gray] at (9.3, 5.7) {Distribution Space};
    
        \draw[thick, fill=blue!2, draw=blue!15] (5, 1.8) ellipse (4.2cm and 1.2cm);
        \node[blue!60!black] at (7.7, 0.2) {Variational Family};
    
        \node[circle, fill=red, inner sep=2pt] (P) at (1.5, 5) {};
        \node[above=2pt of P, align=center] {True Posterior};
        
        \node[circle, fill=blue!80!black, inner sep=2pt] (Qstar) at (3.5, 2.5) {};
        \node[below left=4pt and -15pt of Qstar, align=center, blue!80!black, font=\footnotesize] {Optimal\\Approximation};
        
        \node[circle, fill=orange!90!black, inner sep=2pt] (Qphi) at (7.5, 2.0) {};
        \node[below=4pt and -2cm of Qphi, orange!90!black] {Amortized Solution};
    
        \draw[<->, thick, red!70!black, dashed] (P) -- (Qstar) 
            node[midway, left=6pt, align=right, font=\footnotesize] {Approximation\\Gap};
    
        \draw[<->, thick, orange!90!black] (Qstar) -- (Qphi) 
            node[midway, above, sloped, font=\footnotesize\itshape] {Amortization Gap};
    
        \draw[<->, thin, gray!90, dotted] (P) to [bend left=10] (Qphi);
        \node[gray, font=\footnotesize] at (5.3, 4.2) {Total Gap};
    
    \end{tikzpicture}
    \caption{Schematic decomposition of the inference error.
    The approximation gap arises from the limited expressiveness of the variational family, which may not contain the true posterior.
    The amortization gap captures the additional suboptimality introduced by the inference model, which does not recover the optimal variational parameters within that family.}
    \label{fig:amortization_gap_diagram}
\end{figure}
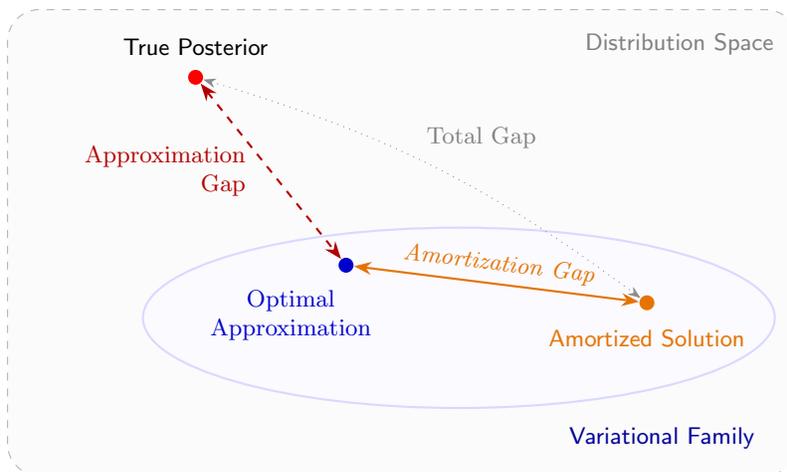

Many works address the amortization gap in VAEs by introducing iterative refinement of variational parameters, combining amortized initialization with input-specific updates~\cite{hoffman2013stochastic,marino2018iterative,kim2018semi}. While effective, these methods increase computational cost and reintroduce optimization for each input data. This raises the question of whether more flexible parameterization can instead be achieved directly at the level of the inference model parameters, avoiding explicit iterative refinement.

The idea of generating model parameters in a data-dependent manner has a substantial history in the literature. 
Early work on fast weights~\cite{schmidhuber1992learning} explored mechanisms in which rapidly changing parameters are generated or updated based on the current input or hidden state, enabling context-dependent adaptation.
Related ideas were later developed in programmable neural networks~\cite{donnarumma2012programming}, proposed by one of the authors of this paper, where a fixed-weight network is augmented with auxiliary inputs that encode a \emph{program} and induce different effective behaviors. In this formulation, a single architecture can emulate a family of networks depending on externally provided control signals.
These ideas were later formalized and popularized under the framework of hypernetworks~\cite{ha2016hypernetworks}, where a neural network generates the parameters of a target network in a input-specific manner. 

In this work, we propose a hypernetwork-based parametrization of amortized variational inference, in which a hypernetwork generates input-specific parameter modulations of a shared base inference model.
We seek to relax the structural constraint imposed by a globally shared inference parametrization within the VAE framework.
Rather than adopting iterative optimization or modeling uncertainty over global weights, we investigate whether instance-specific adaptation can be achieved directly at the parameter level while preserving end-to-end differentiability and computational efficiency.
Concretely, a shared base inference model provides a global parametrization, while a hypernetwork conditioned on each input observation produces modulations of its parameters.

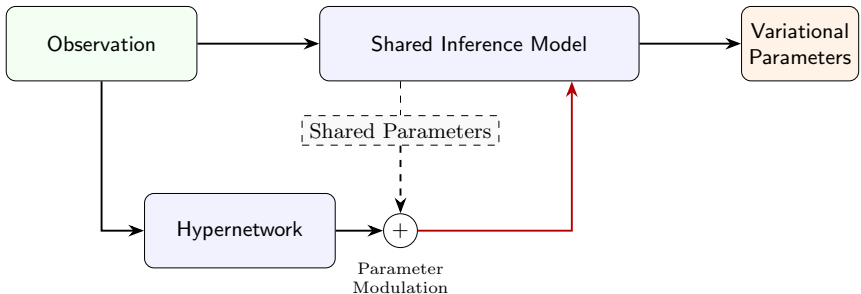
\begin{figure}[h]
    \centering
    \scalebox{0.9}{
        \begin{tikzpicture}[
            node distance=1.5cm and 1.2cm,
            block/.style={rectangle, draw, fill=blue!5, text width=2.6cm, align=center, minimum height=1.1cm, rounded corners, font=\small\sffamily},
            wideblock/.style={block, text width=4.5cm}, 
            dist/.style={rounded corners, draw, fill=orange!10, minimum size=1.1cm, font=\small\sffamily, align=center},
            param/.style={rectangle, draw, dashed, fill=gray!5, font=\footnotesize},
            arrow/.style={-Stealth, thick},
            sum/.style={circle, draw, inner sep=0pt, minimum size=0.5cm},
        ]
            \node (input) [block, fill=green!5] {Observation};
            
            \node (encoder) [wideblock, right=1.8cm of input] {Shared Inference Model};
            
            \node (param_theta) [param, below=0.5cm of encoder.south west, xshift=1.2cm] {Shared Parameters};
            
            \node (latent) [dist, right=1.5cm of encoder] {Variational\\Parameters};    
            
            \node (sum_node) [sum, below=1cm of param_theta] {$+$};
            \node (hyper) [block, left=of sum_node, xshift=0.5cm] {Hypernetwork};
    
            \draw [arrow] (input) -- (encoder);
            \draw [arrow] (encoder) -- (latent);
    
            \draw [dashed] (encoder.south west) ++(1.2cm,0) -- (param_theta.north);
            \draw [arrow, dashed] (param_theta.south) -- (sum_node.north);
    
            \draw [arrow] (input.south) |- (hyper.west);
            \draw [arrow] (hyper.east) -- (sum_node.west);
            \draw [arrow, color=red!70!black] (sum_node.east) -| ($(encoder.south east) - (1.0cm,0)$);
    
            \node [below=0.1cm of sum_node, font=\scriptsize, align=center] {Parameter\\Modulation};
    
        \end{tikzpicture}
    }
    \caption{Schematic representation of the proposed approach. A shared inference model provides a global parametrization, which is adapted through instance-specific parameter modulations generated by a hypernetwork. This enables input-dependent inference while preserving the efficiency of amortization, without requiring iterative optimization.}
    \label{fig:hypernetwork_inference}
\end{figure}

The resulting inference model retains the global structure of the base model while enabling instance-adaptive posterior approximation.
This design preserves the computational advantages of amortized variational inference, as inference remains a single forward pass, while relaxing the rigidity imposed by fully shared parameters. 
By allowing the effective inference parametrization to vary across datapoints, the proposed approach is designed to mitigate the amortization gap without resorting to explicit input-specific optimization steps, thereby relaxing the functional restriction induced by globally shared inference parameters. A representation of the proposed methods is shown in Figure~\ref{fig:hypernetwork_inference}.

The main contributions of this work can be summarized as follows:
\begin{itemize}
    \item We introduce \emph{Instance-Adaptive Variational Autoencoders} (IA-VAE), a hypernetwork-based extension of amortized variational inference that enables input-dependent modulation of a shared inference model.
    
    \item We provide a theoretical analysis showing that, under mild conditions, the variational family induced by IA-VAE contains that of standard amortized inference. We show that it can recover any solution achievable by the base encoder, ensuring that the added flexibility does not compromise the fundamental inference performance.

    \item We empirically demonstrate that IA-VAE improves posterior approximation quality and reduces the amortization gap across both synthetic and real-world datasets, while preserving the computational efficiency of amortized inference.
\end{itemize}

%\newline
%\newline
This work is organized as follows:
in Section~\ref{sec:related_works}, we review related work on amortized variational inference and approaches to mitigate the amortization gap; 
In Section~\ref{sec:proposed_method}, we first provide a brief background on variational inference and its amortized formulation, and then introduce the proposed instance-adaptive parametrization;
in Section~\ref{sec:experiments}, we present the experimental setup, including both synthetic and real image datasets; 
in Section~\ref{sec:results}, we report and analyze the empirical results; 
finally, in Section~\ref{sec:conclusions}, we conclude and discuss future directions.
\section{Related Works}
\label{sec:related_works}
Amortized variational inference in VAEs replaces input-specific optimization with a shared inference network that maps each observation to its approximate posterior parameters. 
While this enables scalable training, it introduces the amortization gap, i.e., the discrepancy between the optimal variational parameters and those produced by the shared inference network. 
The following approaches aim to mitigate this gap by improving the inference mechanism.
Specifically, existing approaches can be broadly categorized into three classes: 
(i) methods that introduce iterative input-specific refinement, 
(ii) approaches that modify the training objective or explicitly model uncertainty in the inference network, and 
(iii) hypernetwork-based methods that generate global model parameters.

\paragraph{(i) Iterative Refinement.}
Early work by~\cite{hjelm2016iterative} proposed an iterative inference framework in which an inference network produces an initial estimate of the variational parameters that is subsequently refined through additional learned update steps.
Building on similar ideas, the authors in~\cite{marino2018iterative} proposed a method based on an encoder that explicitly performs multiple refinement iterations, learning a neural update rule that progressively improves the variational parameters.
Semi-amortized variational autoencoders (SA-VAE)~\cite{kim2018semi} followed a related strategy by augmenting a standard amortized encoder with additional gradient-based refinement steps applied to the variational parameters for each observation. Concretely, the encoder provides an initialization, which is then iteratively improved through stochastic gradient ascent on the ELBO. 

In contrast, our approach achieves adaptation in a single forward pass through a hypernetwork, avoiding iterative refinement. This eliminates the need to unroll input-specific optimization trajectories during training, which in iterative methods leads to increased memory and computational costs that scale with the number of refinement steps.

\paragraph{(ii) Objective-Based and Uncertainty-Aware Methods.}
Several works have investigated the structural sources of amortization suboptimality in VAEs. 
In particular, the authors in~\cite{cremer2018inference} analyzed inference suboptimality and derived a decomposition of the inference error into approximation and amortization components, showing that amortization error arises in part from the shared parametrization of the inference network.
More recently, the authors in~\cite{kim2021reducing} proposed a method that treats the discrepancy between the true posterior and the amortized approximation as uncertainty in the inference model. 
Specifically, the mean and variance functions of the variational posterior are modeled as Gaussian processes~\cite{hastie2009elements}, yielding a stochastic inference network that captures uncertainty in posterior approximation.
This approach reduces amortization error while retaining the efficiency of amortized inference, as posterior parameters are obtained through a single forward pass without input-specific refinement.
Similarly, amortized inference regularization (AIR)~\cite{shu2018amortized} introduces additional regularization on the encoder to control the capacity and smoothness of the amortization family, reshaping the training objective to improve generalization and mitigate inference errors.

In contrast to these approaches, our method leaves the variational objective unchanged and instead increases the expressiveness of the inference parametrization by relaxing the constraint of globally shared parameters.

\paragraph{(iii) Hypernetwork-Based Methods.}
A different line of research explores the generation of model parameters through auxiliary neural networks, commonly referred to as hypernetworks~\cite{ha2016hypernetworks}. 
In this framework, a secondary network produces the weights of a target network in a data-dependent manner, thereby enabling more flexible parametrizations.
HyperVAE~\cite{nguyen2021variational} modeled the parameters of the encoder and decoder as random variables generated by a higher-level generative model. 
Specifically, a hyper-level VAE learns a distribution over the parameters of a base VAE, capturing uncertainty in global network weights. 
Although not developed specifically for VAEs, the authors in~\cite{krueger2017bayesian} introduced Bayesian hypernetworks, in which a hypernetwork transforms samples from a simple noise distribution into the weights of a primary neural network, thereby inducing an implicit distribution over network parameters. 
This enables approximate Bayesian inference over weights while maintaining tractable training via backpropagation.
However, in both cases the generated parameters remain global and are not conditioned on individual observations at inference time. 
As a result, these approaches address uncertainty over model parameters rather than amortization suboptimality in variational inference.

Our approach is most closely related to this class, as it leverages a hypernetwork to increase the flexibility of the inference mechanism. 
However, unlike existing hypernetwork-based methods that generate global parameters shared across datapoints, our method produces \emph{instance-specific} adaptations of the inference network. 
This yields input-dependent encoder parameters, enabling more flexibility while preserving a shared global structure.

\section{Proposed Method}
\label{sec:proposed_method}

We now introduce a more flexible formulation of amortized variational inference that allows the inference model to adapt to individual observations, while retaining a shared global structure.
The key idea is to relax the constraint of a single set of encoder parameters by allowing them to vary as a function of the input. To this end, we augment a base inference model with a hypernetwork that produces input-dependent parameter modulations, yielding an observation-specific encoder.

This construction preserves the inductive bias~\cite{hastie2009elements} of amortized inference while increasing the flexibility of the variational approximation. We refer to the resulting framework as \emph{instance-adaptive variational autoencoder} (IA-VAE).
We first recall the variational inference framework and its amortized formulation, and then present the proposed model.

\subsection{Background}
We briefly review variational inference and its amortized formulation in latent variable models, focusing on VAEs, which provide the foundation for our method.

\subsubsection{Variational Inference.}
Let $\mathcal{D}=\{\vect{x}^{(i)}\}_{i=1}^{N}$ be a dataset of $N$ independent observations $\vect{x}^{(i)} \in \mathbb{R}^d$ drawn from an unknown data-generating process. 
We introduce latent variables $\vect{z} \in \mathbb{R}^k$, with typically $k \ll d$, capturing the unobserved structure underlying the data.
The latent variable model is defined through the joint distribution
\begin{equation}
p(\vect{x},\vect{z}; \theta)=p(\vect{x}|\vect{z}; \theta)p(\vect{z}; \theta),
\label{eq:latent_variable_model}
\end{equation}
where $p(\vect{z}; \theta)$ denotes the prior distribution over latent variables and 
$p(\vect{x} | \vect{z}; \theta)$ is the likelihood function.
The parameters $\theta \in \mathbb{R}^p$ denote the global parameters of the latent variable model.
Parameter learning and posterior inference are generally intractable, since both require evaluating the marginal likelihood $p(\vect{x}; \theta) = \int p(\vect{x}, \vect{z}; \theta)d\vect{z}$, which requires integration over the latent variables. 
Typically, this integral admits no closed-form solution and becomes computationally prohibitive in high dimensions~\cite{kingma2013auto,blei2017variational}. 
Consequently, practical approaches rely on approximate inference methods.

Variational inference addresses this intractability by introducing a tractable variational family $\mathcal{Q}$ of distributions over latent variables, and seeking a member $q(\vect{z}; \lambda) \in \mathcal{Q}$, parameterized by variational parameters $\lambda$, that approximates the true posterior $p(\vect{z}|\vect{x}; \theta)$ by minimizing the Kullback–Leibler divergence~\cite{blei2017variational} $D_\mathrm{KL}\!\left( q(\vect{z};\lambda) \,\|\, p(\vect{z} | \vect{x};\theta) \right)$ .
However, this divergence cannot be minimized directly, as it depends on the unknown posterior $p(\vect{z} | \vect{x};\theta)$.
Instead, the marginal log-likelihood admits the decomposition
\begin{equation}
\log p(\vect{x};\theta) = \mathrm{ELBO}(\lambda,\theta,\vect{x}) + D_\mathrm{KL}\!\left( q(\vect{z};\lambda) \,\|\,  p(\vect{z} | \vect{x};\theta) \right),
\label{eq:vi_marginal_loglik}
\end{equation}
where the evidence lower bound (ELBO) is defined as
\begin{equation}
\mathrm{ELBO}(\lambda,\theta,\vect{x})
=
\mathbb{E}_{q(\vect{z};\lambda)}
\big[
\log p(\vect{x} | \vect{z};\theta)
\big]
-
D_\mathrm{KL}\!\left(
q(\vect{z};\lambda)
\,\|\, 
p(\vect{z};\theta)
\right).
\label{eq:vi_elbo}
\end{equation}
Since the Kullback–Leibler divergence is non-negative, the ELBO is a lower bound on the log-marginal likelihood. Maximizing the ELBO in Eq.~\eqref{eq:vi_elbo} with respect to the variational parameters $\lambda$ is therefore equivalent to minimizing the divergence between the variational distribution and the true posterior.

Thus, given a dataset $\mathcal{D}$, variational inference seeks to determine both the generative parameters $\theta$ and a set of local variational parameters $\{\lambda^{(i)}\}_{i=1}^N$ by maximizing $\sum_{i=1}^{N} \mathrm{ELBO}(\lambda^{(i)}, \theta, \vect{x}^{(i)})$, which provides a tractable surrogate of the marginal log-likelihood of the dataset.

\subsection{Amortized Variational Inference and Variational Autoencoders}
In classical variational inference, a distinct set of variational parameters 
$\lambda^{(i)}$ is optimized for each datapoint $\vect{x}^{(i)}$. 
While this approach yields flexible posterior approximations, it becomes computationally demanding in large-scale settings, as it requires solving a separate optimization problem for every observation.

Amortized variational inference alleviates this limitation by replacing input-specific optimization with a shared inference function. 
Instead of treating $\lambda^{(i)}$ as free parameters, they are obtained as the output of a parameterized mapping
\begin{equation}
\lambda^{(i)} = f(\vect{x}^{(i)}; \phi),
\label{eq:avi_functional}
\end{equation}
where $f$ is an inference model with global learnable parameters $\phi \in \mathbb{R}^q$. The resulting variational distribution can therefore be written as
\begin{equation}
q(\vect{z};\lambda^{(i)}) = q\big(\vect{z}; f(\vect{x}^{(i)};\phi)\big),
\label{eq:avi_posterior}
\end{equation}
which is commonly denoted by $q(\vect{z}|\vect{x}^{(i)};\phi)$.
In this formulation, the variational parameters are no longer independently optimized but are produced through by learned function, thereby amortizing the cost of inference across the dataset.

A prominent instance of amortized variational inference is provided by variational autoencoder (VAE)~\cite{kingma2013auto}. 
In this framework, both the generative model $p(\vect{x}|\vect{z};\theta)$ 
% \apicella{(corresponding to the VAE decoder)} E' NEL PERIODO SUCCESSIVO
and the inference model 
$q(\vect{z}|\vect{x};\phi)$ 
% \apicella{(corresponding to the VAE encoder)} E' NEL PERIODO SUCCESSIVO
are parameterized by artificial neural networks. 
The inference network, i.e. the encoder, maps each observation to the parameters of its approximate posterior distribution, while the generative network, i.e. the decoder, maps latent variables to the data space.
The prior distribution over the latent variables is typically chosen to be a standard isotropic Gaussian, $p(\vect{z}) = \mathcal{N}(\vect{z}; \vect{0}, I)$.
For simplicity of notation, in the following we denote the prior as $p(\vect{z})$, omitting the explicit dependence on the parameters $\theta$ whenever no ambiguity arises.

Under this formulation, marginal log-likelihood in Eq.~\eqref{eq:vi_marginal_loglik} depends on both the generative parameters $\theta$ and the inference parameters $\phi$, obtaining
\begin{equation}
\log p(\vect{x};\theta) = \mathrm{ELBO}(\phi,\theta,\vect{x}) + D_\mathrm{KL}\!\left( q(\vect{z}|\vect{x};\phi) \,\|\,  p(\vect{z}|\vect{x};\theta) \right),
\label{eq:vae_marginal_loglik}
\end{equation}
where the ELBO is defined as
\begin{equation}
\mathrm{ELBO}(\phi,\theta,\vect{x})
=
\mathbb{E}_{q(\vect{z}|\vect{x};\phi)}
\big[
\log p(\vect{x} | \vect{z};\theta)
\big]
-
D_\mathrm{KL}\!\left(
q(\vect{z}|\vect{x};\phi)
\,\|\, 
p(\vect{z})
\right).
\label{eq:vae_elbo}
\end{equation}
The objective becomes $\sum_{i=1}^{N} \mathrm{ELBO}\big(\phi, \theta; \vect{x}^{(i)}\big)$, which is jointly optimized with respect to the parameters $\theta$ and $\phi$.

Training is performed by maximizing the ELBO using gradient-based optimization. 
To enable gradient estimation, samples from $q(\vect{z} | \vect{x};\phi)$ are expressed as deterministic transformations of auxiliary noise variables through the \textit{reparameterization trick}~\cite{kingma2013auto}, allowing gradients to propagate through both the encoder and decoder in an end-to-end differentiable manner.

\subsection{Proposed Method: Instance-Adaptive Variational Autoencoder}
Let $q(\vect{z}|\vect{x};\phi_{\mathrm{AVI}})$ denote an encoder with global parameters $\phi_{\mathrm{AVI}}$ trained on a dataset $\mathcal{D}$ via amortized variational inference. 
To introduce instance-specific flexibility, we augment this architecture with a hypernetwork $h(\vect{x};\psi)$ having parameters $\psi \in \mathbb{R}^r$ that generates parameter modulations for each input data. 
To this end, we substitute the encoder $q(\vect{z}|\vect{x};\phi_{\mathrm{AVI}})$  with a new instance-adaptive encoder $q(\vect{z}|\vect{x};\phi(\mathbf{x};\psi))$. 
Specifically, given an observation $\vect{x}$, the hypernetwork produces a parameter adjustment that modifies the parameters of the base inference model, yielding the instance-specific inference parameters as follows: 
\begin{equation}
\label{eq:iavae}
\phi(\vect{x}; \psi) = \phi_{\mathrm{AVI}} + h(\vect{x};\psi).
\end{equation}
This modulation induces an inductive bias toward input-dependent adaptations of the shared encoder, rather than learning entirely independent inference mappings for each datapoint. As a result, the model retains the global structure captured by $\phi_{\mathrm{AVI}}$ while enabling observation-specific adaptation.

Importantly, the proposed method does not alter the variational objective itself; rather, it enlarges the class of inference mappings used to parameterize the variational posterior.
Indeed, the marginal log-likelihood in Eq.~\eqref{eq:vae_marginal_loglik} becomes
\begin{equation}
\label{eq:avi_marginal_loglik}
\log p(\vect{x};\theta) = \mathrm{ELBO}(\psi,\theta,\vect{x}) + D_\mathrm{KL}\!\left( q(\vect{z}|\vect{x};\phi(\vect{x};\psi)) \,\|\,  p(\vect{z}|\vect{x};\theta) \right),
\end{equation}
where the ELBO is defined as
\begin{equation}
\label{eq:iavae_elbo}
\mathrm{ELBO}(\psi,\theta,\vect{x}) =
\mathbb{E}_{q(\vect{z}|\vect{x};\phi(\vect{x};\psi))}
\big[
\log p(\vect{x} | \vect{z};\theta)
\big]
-
D_\mathrm{KL}\!\left(
q(\vect{z}|\vect{x};\phi(\vect{x};\psi))
\,\|\, 
p(\vect{z})
\right).
\end{equation}
Training is performed by maximizing the standard VAE objective 
$\sum_{i=1}^{N} \mathrm{ELBO}\big(\psi, \theta; \vect{x}^{(i)}\big)$, with respect to the parameters $\theta$ and $\psi$.
A graphical representation of the proposal is shown in Figure~\ref{fig:proposal_iavae}.

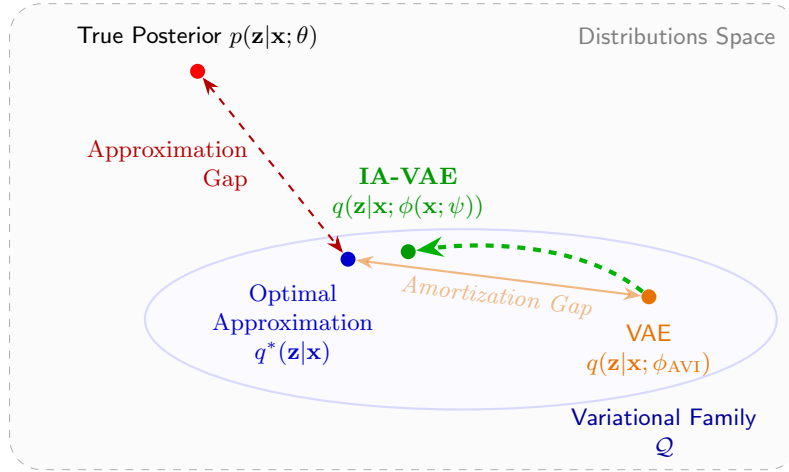
\begin{figure}[t]
    \centering
    \begin{tikzpicture}[
        node distance=1.2cm,
        font=\sffamily\small,
        >=Stealth
    ]
        \draw[fill=gray!3, draw=gray!60, dashed, rounded corners=12pt] (-1, -0.3) rectangle (9.5, 5.9);
        \node[anchor=north east, gray] at (9.3, 5.7) {Distributions Space};
    
        \draw[thick, fill=blue!2, draw=blue!15] (5, 1.7) ellipse (4.2cm and 1.2cm);
        \node[blue!60!black, align=center] at (7.7, 0.2) {Variational Family \\ $\mathcal{Q}$};
    
        \node[circle, fill=red, inner sep=2pt] (P) at (1.5, 5) {};
        \node[above=2pt of P, align=center] {True Posterior  $p(\vect{z}|\vect{x}; \theta)$};
        
        \node[circle, fill=blue!80!black, inner sep=2pt] (Qstar) at (3.5, 2.5) {};
        \node[below left=4pt and -15pt of Qstar, align=center, blue!80!black, font=\footnotesize] {Optimal\\Approximation\\$q^*(\vect{z}|\vect{x})$};
        
        \node[circle, fill=orange!90!black, inner sep=2pt] (Qphi) at (7.5, 2.0) {};
        \node[below=4pt and -2cm of Qphi, orange!90!black, align=center] {VAE \\ $q(\vect{z}|\vect{x};\phi_{\mathrm{AVI}})$};
    
        \draw[<->, thick, red!70!black, dashed] (P) -- (Qstar) 
            node[midway, left=6pt, align=right, font=\footnotesize] {Approximation\\Gap};
    
        \node[circle, fill=green!60!black, inner sep=2pt] (Qhyper) at (4.3, 2.6) {};
        \node[above=6pt of Qhyper, green!60!black, align=center, font=\footnotesize\bfseries] (Lhyper) {IA-VAE \\ $q(\vect{z}|\vect{x}; \phi(\vect{x};\psi))$};

         \draw[->, ultra thick, green!70!black, dashed] (Qphi) .. controls (6.5, 2.8) and (5.0, 2.7) .. (Qhyper);

         \draw[<->, thick, orange!90!black, opacity=0.5] (Qstar) -- (Qphi) 
            node[midway, below, sloped, font=\footnotesize\itshape] {Amortization Gap};
    
    \end{tikzpicture}
    \caption{Conceptual illustration of the proposed IA-VAE method. 
    Starting from the amortized solution $q(\vect{z}|\vect{x};\phi_{\mathrm{AVI}})$, input-dependent parameter modulations adapt the encoder for each observation, yielding an instance-adaptive encoder $q(\vect{z}|\vect{x};\phi(\vect{x};\psi))$. This induces a shift of the variational distribution toward the optimal approximation $q^*(\vect{z}|\vect{x})$ within the variational family.}
    \label{fig:proposal_iavae}
\end{figure}

\paragraph{Augmenting amortized variational inference with IA-VAE.}
The proposed construction can be interpreted as an extension of a pretrained amortized inference model, in which a fixed set of encoder parameters is augmented with input-dependent modulations generated by a hypernetwork.
We now formalize the relationship between the variational family induced by the base amortized encoder and the one induced by IA-VAE.

Let $\mathcal{X} \subseteq \mathbb{R}^d$ denote the input space, and let $\Psi$ be the parameter space of the hypernetwork.
Consider a pretrained amortized inference model $q(\vect{z}|\vect{x};\phi_{\mathrm{AVI}})$ with fixed parameters $\phi_{\mathrm{AVI}}$.
We define the corresponding variational family as
\begin{equation}
\mathcal{Q}_{\mathrm{AVI}}
:=
\{
q(\vect{z}|\vect{x}; \phi_{\mathrm{AVI}})
:\vect{x} \in \mathcal X
\}.
\label{eq:q_avi_fixed_map}
\end{equation}
In IA-VAE, the inference parameters are modulated as in Eq.~\eqref{eq:iavae}, yielding the variational family
\begin{equation}
\mathcal{Q}_{\mathrm{IA\mbox{-}VAE}}
:=
\left\{
q(\vect{z}|\vect{x}; \phi_{\mathrm{AVI}} + h(\vect{x};\psi))
:
\vect{x} \in \mathcal X, \psi \in \Psi
\right\}.
\label{eq:q_iavae_family_fixed}
\end{equation}

\begin{proposition}
\label{prop:avi_subset_iavae}
Assume that there exists a non-empty subset $\Psi_0 \subseteq \Psi$ such that $h(\vect{x};\psi)=\vect{0}, \forall \vect{x}\in\mathcal X,\ \forall \psi \in \Psi_0$.
Then $\mathcal{Q}_{\mathrm{AVI}} \subseteq \mathcal{Q}_{\mathrm{IA\mbox{-}VAE}}$.
\end{proposition}

\begin{proof}
Take any element of $\mathcal{Q}_{\mathrm{AVI}}$.
By definition, it can be written as $q(\vect{z}|\vect{x}; \phi_{\mathrm{AVI}})$ for some $\vect{x} \in \mathcal X$.
Since $\Psi_0$ is non-empty, choose any $\psi_0 \in \Psi_0$.
By assumption, $h(\vect{x};\psi_0)=\vect{0}, \forall \vect{x} \in \mathcal{X}$.
Therefore,
\begin{equation}
\label{eq:theorem_zero_modulation}
\phi_{\mathrm{AVI}} + h(\vect{x};\psi_0)
=
\phi_{\mathrm{AVI}},
\end{equation}
and hence
\begin{equation}
\label{eq:theorem_same_solution}
q(\vect{z}|\vect{x}; \phi_{\mathrm{AVI}} + h(\vect{x};\psi_0))
=
q(\vect{z}|\vect{x}; \phi_{\mathrm{AVI}}).
\end{equation}
It follows that every element of $\mathcal{Q}_{\mathrm{AVI}}$ also belongs to $\mathcal{Q}_{\mathrm{IA\mbox{-}VAE}}$, which proves the claim.
\end{proof}

\begin{remark}
A sufficient condition for the assumption of Proposition~\ref{prop:avi_subset_iavae} to hold is that the hypernetwork $h(\vect{x};\psi)$ admits parameter configurations that yield identically zero output for all inputs. 
For instance, this is satisfied if the weights and biases of the output layer are set to zero, regardless of the values of the preceding layers. Hence, in typical neural network parameterizations, the set $\Psi_0$ is non-empty.
\end{remark}

\begin{corollary}
\label{cor:elbo_pointwise}
Under the assumptions of Proposition~\ref{prop:avi_subset_iavae}, for any fixed generative parameters $\theta$ and any observation $\vect{x} \in \mathcal{X}$,
\[
\mathrm{ELBO}(\phi_{\mathrm{AVI}},\theta,\vect{x})
\leq
\sup_{\psi \in \Psi}
\mathrm{ELBO}(\psi,\theta,\vect{x}).
\]
\end{corollary}

\begin{proof}
By Proposition~\ref{prop:avi_subset_iavae}, there exists $\psi_0 \in \Psi$ such that Eq.~\eqref{eq:theorem_same_solution} holds.
Therefore, the corresponding ELBO values coincide:
\begin{equation}
\mathrm{ELBO}(\phi_{\mathrm{AVI}},\theta,\vect{x})
=
\mathrm{ELBO}(\psi_0,\theta,\vect{x}).
\end{equation}
Since $\psi_0 \in \Psi$, we obtain
\begin{equation}
\label{eq:corollary_sup}
\mathrm{ELBO}(\phi_{\mathrm{AVI}},\theta,\vect{x})
\leq
\sup_{\psi \in \Psi}
\mathrm{ELBO}(\psi,\theta,\vect{x}),
\end{equation}
which proves the claim.
\end{proof}

\subsubsection{Discussion.}
The above results provide a formal justification for IA-VAE as a principled extension of amortized variational inference. 
In particular, Proposition~\ref{prop:avi_subset_iavae} shows that, under an assumption satisfied by standard neural network parameterizations, the IA-VAE variational family contains the amortized one.
As a consequence, IA-VAE preserves all solutions achievable by standard amortized inference, since the base encoder is recovered as a special case corresponding to zero modulation.

From an optimization perspective, Corollary~\ref{cor:elbo_pointwise} establishes that the best ELBO obtainable within IA-VAE is at least as high as that of standard amortized inference.
This result strongly motivates the proposed construction, as IA-VAE retains the efficiency of amortization while enabling more flexible inference mappings.

Importantly, these guarantees are purely representational and do not, by themselves, ensure improved performance after training.
Nevertheless, they establish IA-VAE as a more expressive extension of amortized inference that retains its baseline solutions. In practice, this additional flexibility can be leveraged to obtain tighter variational approximations, as demonstrated by the empirical results.

\subsubsection{Blockwise Parameter Modulation.}
To ensure scalability, the hypernetwork does not generate the modulation of the entire parameter set of the inference model in a single step. 
Instead, following~\cite{ha2016hypernetworks}, parameter modulations are generated in blocks corresponding to different components of the inference model. 
Specifically, each block of parameters $\ell$ is associated with a learnable embedding $\vect{e}_\ell \in \mathbb{R}^l$ that identifies the corresponding part of the inference model. 
Given an observation $\vect{x}$ and an embedding $\vect{e}_\ell$, the hypernetwork $h(\cdot;\psi)$ produces the parameter modulation for that block.
This mechanism allows a single hypernetwork to generate parameters for inference models of arbitrary size without requiring the hypernetwork itself to mirror the full size of the base inference model.
Formally, for the $\ell$-th parameter block we generate
\begin{equation}
\phi_\ell(\vect{x}; \psi) = \phi_{\mathrm{AVI}, \ell} + h(\vect{x}, \vect{e}_\ell;\psi).
\label{eq:iavae_embeddings}
\end{equation}
A graphical representation of the IA-VAE architecture is shown in Figure~\ref{fig:iavae_details}.
The full procedure of IA-VAE is shown in Algorithm~\ref{alg:IA-VAE}. 

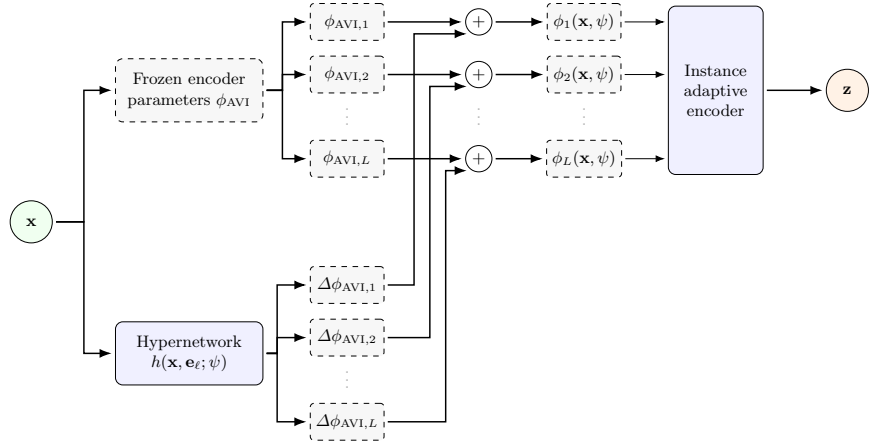
\begin{figure*}[t]
\centering
\scalebox{0.7}{
\begin{tikzpicture}[
    font=\small,
    >=Latex,
    node distance=0.4cm and 1.2cm,
    bigbox/.style={draw, rounded corners=4pt, align=center, minimum height=1.2cm, minimum width=2.8cm},
    block/.style={draw, rounded corners=2pt, align=center, minimum height=0.7cm, minimum width=1.4cm},
    op/.style={draw, circle, minimum size=0.5cm, inner sep=0pt},
    data/.style={draw, circle, minimum size=0.8cm, inner sep=0pt, align=center},
    note/.style={font=\scriptsize, align=center, gray!80},
    every node/.style={outer sep=2pt}
]

% --- INPUT ---
\node[data, fill=green!6] (x) at (0,1) {$\mathbf{x}$};

% --- BRANCH SUPERIORE (Base Encoder) ---
\node[bigbox, fill=blue!6, dashed, fill=gray!6] (base) at (3, 3.5) {Frozen encoder\\parameters $\phi_{\mathrm{AVI}}$};

\node[block, fill=blue!2, dashed, fill=gray!6] (b1) at (6, 4.8) {$\phi_{\mathrm{AVI},1}$};
\node[block, fill=blue!2, dashed, fill=gray!6] (b2) at (6, 3.8) {$\phi_{\mathrm{AVI},2}$};
\node[gray!60] (b_dots) at (6, 3.1) {$\vdots$}; 
\node[block, fill=blue!2, dashed, fill=gray!6] (b3) at (6, 2.2) {$\phi_{\mathrm{AVI},L}$};

\draw[->, thick] (base.east) -- ++(0.3,0) |- (b1.west);
\draw[->, thick] (base.east) -- ++(0.3,0) |- (b2.west);
\draw[->, thick] (base.east) -- ++(0.3,0) |- (b3.west);

% --- BRANCH INFERIORE (Hypernetwork) ---
\node[bigbox, fill=blue!6] (hyper) at (3, -1.5) {Hypernetwork\\$h(\mathbf{x},\mathbf{e}_\ell;\psi)$};
\draw[->, thick] (x) -- ++(1,0) |- (hyper.west);
\draw[->, thick] (x) -- ++(1,0) |- (base.west);
% Outputs Hypernetwork
\node[block, dashed, fill=gray!6] (h1) at (6, -0.2) {$\Delta \phi_{\mathrm{AVI},1}$};
\node[block, dashed, fill=gray!6] (h2) at (6, -1.2) {$\Delta \phi_{\mathrm{AVI},2}$};
\node[gray!60] (e_dots) at (6, -1.9) {$\vdots$};
\node[block, dashed, fill=gray!6] (h3) at (6, -2.8) {$\Delta \phi_{\mathrm{AVI},L}$};

\draw[->, thick] (hyper.east) -- ++(0.2,0) |- (h1.west);
\draw[->, thick] (hyper.east) -- ++(0.2,0) |- (h2.west);
\draw[->, thick] (hyper.east) -- ++(0.2,0) |- (h3.west);

% --- SOMMATORI ---
\node[op] (p1) at (8.5, 4.8) {$+$};
\node[op] (p2) at (8.5, 3.8) {$+$};
\node[gray!60] (p_dots) at (8.5, 3.1) {$\vdots$};
\node[op] (p3) at (8.5, 2.2) {$+$};

\draw[->, thick] (b1) -- (p1);
\draw[->, thick] (b2) -- (p2);
\draw[->, thick] (b3) -- (p3);

% Frecce verticali dai blocchi h ai sommatori (EVITA INCROCI)
\draw[->, thick] (h1.east) -- ++(0.5,0) |- (p1.south west);
\draw[->, thick] (h2.east) -- ++(0.8,0) |- (p2.south west);
\draw[->, thick] (h3.east) -- ++(1.1,0) |- (p3.south west);

% --- PARAMETRI FINALI ---
\node[block, dashed, fill=gray!6] (phi1) at (10.5, 4.8) {$\phi_1(\mathbf{x}, \psi)$};
\node[block, dashed, fill=gray!6] (phi2) at (10.5, 3.8) {$\phi_2(\mathbf{x}, \psi)$};
\node[gray!60] (phi_dots) at (10.5, 3.1) {$\vdots$};
\node[block, dashed, fill=gray!6] (phi3) at (10.5, 2.2) {$\phi_L(\mathbf{x}, \psi)$};

\draw[->, thick] (p1) -- (phi1);
\draw[->, thick] (p2) -- (phi2);
\draw[->, thick] (p3) -- (phi3);

% --- ENCODER ADATTIVO ---
\node[bigbox, fill=blue!6, minimum width=1.8cm, minimum height=3.2cm] (enc) at (13, 3.5) {Instance\\adaptive\\encoder};

\draw[->] (phi1.east) -- (enc.west|-phi1.east);
\draw[->] (phi2.east) -- (enc.west|-phi2.east);
\draw[->] (phi3.east) -- (enc.west|-phi3.east);

\node[data, fill=orange!10] (z) at (15.5, 3.5) {$\mathbf{z}$};

\draw[->, thick] (enc.east) -- ++(0.3,0) |- (z.west);

\end{tikzpicture}}
\caption{Architectural illustration of the proposed IA-VAE framework. 
A frozen encoder with global parameters $\phi_{\mathrm{AVI}}$ is augmented by a hypernetwork $h(\vect{x}, \vect{e};\psi)$ that generates input-dependent parameter modulations $\Delta \phi_{\mathrm{AVI, \ell}}$ for each parameter block $\ell$, with $1 \le \ell \le L$. 
These modulations are combined with the corresponding base parameters to produce instance-specific encoder parameters $\phi(\mathbf{x}, \psi)$, yielding an adaptive inference encoder with enhanced flexibility in the variational posterior.}
\label{fig:iavae_details}
\end{figure*}

\begin{algorithm}[t]
\caption{Instance-Adaptive Amortized Variational Inference}
\label{alg:IA-VAE}
\begin{algorithmic}
\STATE \textbf{Input:} frozen base inference parameters $\phi_{\mathrm{AVI}}$, generative parameters $\theta$, hypernetwork parameters $\psi$, learnable block embeddings $\{e_\ell\}_{\ell=1}^{L}$, Monte Carlo samples $S$, dataset $\mathcal{D}$
\STATE
\STATE Sample minibatch $\{\vect{x}^{(i)}\}_{i=1}^{B} \sim \mathcal{D}$
\FOR{$i=1$ to $B$}
    \FOR{$\ell=1$ to $L$}
        \STATE $\phi_\ell^{(i)} \leftarrow \phi_{\mathrm{AVI},\ell} + h(\vect{x}^{(i)}, \vect{e}_\ell;\psi)$
    \ENDFOR
    \STATE $\phi^{(i)} \leftarrow \{\phi_\ell^{(i)}\}_{\ell=1}^{L}$
    \FOR{$s=1$ to $S$}
        \STATE Sample $\vect{z}^{(s)} \sim q(\vect{z} | \vect{x}^{(i)};\phi^{(i)})$ using the reparameterization trick
        \STATE $\mathcal{L}_s^{(i)} \leftarrow \log p(\vect{x}^{(i)} | \vect{z}^{(s)};\theta)$
    \ENDFOR
    \STATE $\mathcal{L}^{(i)} \leftarrow \frac{1}{S} \sum_{s=1}^S \mathcal{L}_s^{(i)} - D_\mathrm{KL}\!\left(q(\vect{z} | \vect{x}^{(i)};\phi^{(i)}) \,\|\, p(\vect{z})\right)$
\ENDFOR
\STATE $\mathcal{L} \leftarrow \frac{1}{B}\sum_{i=1}^{B}\mathcal{L}^{(i)}$
\STATE Update $\theta$, $\psi$, and $\{e_\ell\}_{\ell=1}^{L}$ based on $\nabla_{\theta}\mathcal{L}$, $\nabla_{\psi}\mathcal{L}$, and $\nabla_{\{\vect{e}_\ell\}}\mathcal{L}$
\end{algorithmic}
\end{algorithm}

\section{Experimental Setup}
\label{sec:experiments}
In this section, we describe the experimental setup used to evaluate IA-VAE.
We consider two complementary experimental scenarios: a synthetic dataset and image datasets.

The synthetic setting is designed to provide a controlled environment in which the true generative process is fully known. 
This allows us to isolate the amortization gap and directly assess the quality of the inferred posterior, disentangling inference errors from modeling errors in the decoder. 
In contrast, experiments on image datasets evaluate the proposed method in realistic high-dimensional settings, where the true generative process is unknown and such direct analyses are not feasible.

Our primary objective is to assess to what extent instance-adaptive inference can reduce the amortization gap. 
To this end, we consider multiple evaluation criteria.
We first measure improvements in terms of ELBO, which quantifies the overall quality of the learned latent representations and is applicable in both experimental settings.
In addition, in the synthetic setting, we leverage access to the true generative model to directly evaluate the accuracy of the inferred posterior. 
Specifically, we compare inferred representations against reference quantities such as the maximum a posteriori (MAP) estimate and the local geometry of the true posterior. 
These analyses provide a more fine-grained assessment of inference quality, allowing us to determine whether improvements in ELBO correspond to more accurate approximations of the true posterior.

Furthermore, still in the controlled synthetic setting, we investigate robustness to random initialization and parameter efficiency. 
This is particularly meaningful in this regime, where the decoder is fixed to the true generative model and does not introduce additional sources of variability. 
As a result, differences in performance can be more directly attributed to the inference mechanism itself. 
In this context, robustness across random initializations allows us to assess the stability of the method, while parameter efficiency analysis helps determine whether similar improvements could be obtained by simply increasing the capacity of a standard encoder.

In the synthetic setting, in order to further isolate the amortization gap, we do not introduce a train/validation/test split. 
All evaluations are performed on the same dataset used for training, as our goal is not to assess generalization but to analyze inference quality under controlled conditions, avoiding additional variability induced by data partitioning.

In contrast, in the image data experiments, evaluation is necessarily restricted to ELBO computed on held-out data, which serve as proxies for inference quality. 
For each dataset, all experiments are repeated across multiple random seeds. 
In addition, we perform statistical hypothesis testing to assess the significance of the observed performance differences.
All models are trained on the standard training split of each dataset, with model selection performed on a validation set obtained by randomly sampling $15\%$ of the training data, and final performance reported on the predefined test sets, following a strict protocol to prevent data leakage, as defined in~\cite{apicella2025don}.

In the following, we describe the datasets, model configurations, and evaluation protocols adopted in each regime. 
A summary of all experimental settings and hyperparameters is reported in Table~\ref{tab:hparams_all}.

\begin{table}[t]
\centering
\caption{Summary of hyperparameters and experimental settings for synthetic and image data experiments.}
\scalebox{0.8}{
    \begin{tabular}{lll}
    \toprule
    \textbf{Component} & \textbf{Synthetic Data} & \textbf{Image Data} \\
    \midrule
    
    \multicolumn{3}{c}{\textit{Dataset}} \\
    \midrule
    Dataset & Synthetic ($N=5000$) & OMNIGLOT, MNIST, Fashion MNIST \\
    Latent dim & 2 & 32 \\
    Observation dim & 3 & $28 \times 28$ (binarized) \\
    Likelihood & Gaussian ($\sigma=0.1$) & Bernoulli \\
    Train / Test split & -- & Standard protocol \\
    Train / Val split & -- & 85\% / 15\% \\
    
    \midrule
    \multicolumn{3}{c}{\textit{IA-VAE Architecture}} \\
    \midrule
    VAE encoder & MLP (1 hidden layer, 2 units, ReLU) & 3-layer ResNet \\
    VAE decoder & Fixed true mapping & 12-layer Gated PixelCNN \\
    Hypernetwork & Single linear projection & Structured (conv: slice-wise, FC: column-wise) \\
    $\mathbf{e}_\ell$ dim & 2 & 16 \\
    $\mathbf{e}^{\text{in}}$ dim & -- & 8 \\
    
    \midrule
    \multicolumn{3}{c}{\textit{Optimization}} \\
    \midrule
    Optimizer & Adam & As in \cite{kim2018semi} \\
    Learning rate & $1 \times 10^{-4}$ & $1 \times 10^{-4}$ (hypernetwork), $5 \times 10^{-5}$ (decoder) \\
    Batch size & 32 & As in \cite{kim2018semi} \\
    Epochs & 1000 & As in \cite{kim2018semi} \\
    Early stopping & Patience = 100 & As in \cite{kim2018semi} \\
    Weight decay & -- & $5 \times 10^{-3}$ (OMNIGLOT), $5 \times 10^{-4}$ (MNIST) \\
    
    \midrule
    \multicolumn{3}{c}{\textit{Evaluation}} \\
    \midrule
    Variational metrics & ELBO, $D_{\mathrm{KL}}$ & ELBO \\
    Posterior accuracy & $d_{\mathrm{MAP}}$, $r_{\mathrm{MAP}}$ & -- \\
    Qualitative analysis & Posterior visualization & Reconstruction visualization \\
    Robustness & Random initializations & Random initializations \\
    Parameter efficiency & Encoder capacity scaling & -- \\
    
    \bottomrule
    \end{tabular}
}
\label{tab:hparams_all}
\end{table}

\subsection{Synthetic Data}
Following~\cite{kim2018semi}, we first evaluate our approach on a synthetic dataset where the true generative model is known. Let $\vect{z} \in \mathbb{R}^2$ denote a latent variable drawn from a standard Gaussian prior and $\vect{x} \in \mathbb{R}^3$ the corresponding observation. The generative process is defined as
\[
\vect{z} \sim \mathcal{N}(0, I_2), \qquad 
\vect{x} = f(\vect{z}) + \epsilon, \qquad 
\epsilon \sim \mathcal{N}(0, \sigma^2 I_3).
\]
The deterministic mapping $f : \mathbb{R}^2 \rightarrow \mathbb{R}^3$ is defined as
\[
f(\vect{z}) = A\vect{z} + g(\vect{z}),
\]
where $A \in \mathbb{R}^{3 \times 2}$ is a fixed full-rank matrix and $g(\vect{z})$ introduces a nonlinear interaction between the latent dimensions. Specifically, we define
\[
A =
\begin{bmatrix}
1 & 0 \\
0 & 1 \\
1 & 1
\end{bmatrix},
\qquad
g(\vect{z}) =
\begin{bmatrix}
0 \\
0 \\
z_1 z_2
\end{bmatrix}.
\]
The observation noise level is set to $\sigma = 0.1$, resulting in posterior distributions that are well concentrated around the true latent variables while retaining a non-degenerate level of uncertainty. A dataset of $5000$ samples is generated following this procedure.

In order to isolate the effect of the amortization gap, the decoder is fixed to the true generative mapping $f(\vect{z})$ used to generate the synthetic dataset. In this setting, the only source of approximation arises from the amortized inference network, since the additional complexity associated with learning the decoder is removed.

Moreover, since the goal of this experiment is to isolate and analyze the amortization gap under controlled conditions, we do not consider a held-out validation or test set. 
Instead, all evaluations are performed on the same dataset used for training. 
This choice avoids introducing additional sources of variability related to generalization, allowing us to focus exclusively on the quality of the inferred posterior.

To evaluate IA-VAE, we first train a VAE whose encoder is parameterized by multilayer perceptrons with a single hidden layer of 2 neurons and ReLU as activation functions. 
Given the simplicity of the VAE architecture, we adopt a lightweight hypernetwork for the IA-VAE, where parameter modulations for each encoder block are generated through a single linear projection. 
This projection is kept linear, following the original implementation in~\cite{ha2016hypernetworks} 
Specifically, considering $d_\ell$ as the number of elements of the $\ell$-th block and $d_\text{max}$ as the maximum number of elements across all blocks, 
the hypernetwork takes as input the datapoint $\vect{x}$ together with the embedding $\vect{e}_\ell$ of dimensions $l=2$, and is defined as
\[
h(\vect{x}, \vect{e}_\ell; \psi) = W [\vect{x}; \vect{e}_\ell] + b,
\]
where $[\vect{x}; \vect{e}_\ell]$ denotes the concatenation of the observation $\vect{x}$ and the block embedding $\vect{e}_\ell$, and $\psi = \{W, b\}$ represents the hypernetwork parameters, with $W \in \mathbb{R}^{d_\text{max} \times (d+l)}$ and $b \in \mathbb{R}^{d_\text{max}}$. 
For each block $\ell$, only the first $d_\ell$ components of this vector are used to construct the parameter modulation. Thus, the $\ell$-th parameter block is computed as
\[
\phi_\ell(\vect{x}; \psi) = \phi_{\mathrm{AVI}, \ell} + h(\vect{x}, \vect{e}_\ell;\psi)_{1:d_\ell}.
\]
All the experiments were performed using the Adam optimizer with a learning rate of $1 \times 10^{-4}$ for a maximum of $1000$ training epochs, with a batch size of $32$. 
Early stopping is applied by monitoring the ELBO, and training is terminated if no improvement is observed for $100$ consecutive epochs.
The hypernetwork parameters $W$ and $b$ are initialized with zero-mean Gaussian weights with standard deviation $10^{-3}$ and zero bias, so that the model starts close to the base encoder parameters and progressively learns instance-specific modulations.

We analyze the behavior of IA-VAE along the following four complementary aspects: 

\subsubsection{Robustness across random initializations:} 
Starting from a trained VAE, we train multiple IA-VAE models to account for the stochasticity induced by random initialization and random components of the training process. Specifically, we train 10 independent IA-VAE models initialized with different random seeds, and report the average metrics across these runs.
To also account for variability due to the initialization of the base VAE itself, the entire procedure is repeated across 10 independently trained VAE models initialized with different random seeds. 

\subsubsection{Qualitative posterior analysis:} 
To qualitatively assess the inference behavior, we visualize the posterior density for two randomly selected observations. 
In particular, we will show the true latent variable used to generate the observation, the MAP estimate, and the posterior means inferred by the VAE and the proposed IA-VAE, in order to assess whether the inferred posterior means lie in regions of high posterior probability and accurately capture the structure of the true posterior.

\subsubsection{Quantitative posterior evaluation:} 
To quantitatively assess the accuracy of the inferred latent representations, we measure their discrepancy from the MAP estimate while accounting for the local geometry of the posterior distribution. 
The MAP estimate serves as a reference point, as it identifies the most probable latent configuration under the true posterior. 
To this end, we approximate the posterior locally with a Gaussian distribution fitted to the posterior density. 
The covariance of this approximation captures the local uncertainty structure of the posterior and allows us to evaluate distances using the Mahalanobis distance~\cite{hastie2009elements}, thus providing a scale-aware measure of how far an inferred latent representation lies from the MAP relative to the posterior spread. 
This also enables the computation of density-based quantities for the inferred posterior means. 

\subsubsection{Parameter efficiency:}
To assess whether the improvements obtained by IA-VAE could be explained simply by increasing the capacity of a standard amortized encoder, we trained a family of baseline VAEs with a progressively wider hidden layer. In these models, the encoder architecture remains identical except for the size of the hidden layer, resulting in parameter counts ranging from 20 to 164 parameters.

% \subsubsection{Trainable decoder: 
% We now consider a more realistic setting in which the decoder parameters are learned jointly with the inference network rather than fixed to the ground-truth generative mapping. 
% In this case, the decoder is parameterized by a multilayer perceptron with an architecture symmetric to the encoder.  
% We therefore compare the VAE and IA-VAE in terms of ELBO and its KL component over the dataset.

\subsection{Image Data}

To evaluate the proposed approach on non-synthetic data, we conduct experiments on image datasets following the experimental setup reported in~\cite{kim2018semi}.
Unless explicitly stated otherwise, all hyperparameters and training details are identical to those reported therein.
Using this setup, we conducted experiments on OMNIGLOT~\cite{lake2015human}, MNIST~\cite{lecun2002gradient} and Fashion MNIST~\cite{xiao2017fashion} datasets. 

Following~\cite{kim2018semi}, the input images are binarized and modeled using a Bernoulli observation likelihood. 
The encoder is a 3-layer ResNet that outputs the parameters of a Gaussian variational posterior with diagonal covariance and latent dimensionality $32$. 
The decoder is a 12-layer Gated PixelCNN decoder with a Bernoulli likelihood.

In contrast to the synthetic setting, where the true generative process is known and more detailed posterior analyses are possible, evaluation on image datasets is necessarily restricted to variational quantities.
In particular, we assess model performance in terms of the ELBO and its KL divergence component, computed on held-out data.
All models are trained on the training split of each dataset, while model selection is performed on a validation set based on the ELBO.
The validation set is obtained by randomly sampling $15\%$ of the training data, and final results are reported on the standard test split provided with each dataset.

All experiments are repeated across 10 independent runs with different random seeds. 
We report mean and standard deviation of the ELBO across runs, and assess the statistical significance of the observed performance differences via hypothesis testing procedure.

For the IA-VAE, in contrast to the experiments on synthetic data where the hypernetwork reduces to a simple linear projection, we adopt a more complex hypernetwork to match the higher complexity of the encoder, while still adhering to the blockwise parameter modulation strategy.
In particular, for convolutional layers we follow~\cite{ha2016hypernetworks} exactly, using the same architecture and generating kernel weights modulations slice-by-slice along the input-channel dimension.

For fully connected layers, however, directly applying the same parameterization would lead to a substantially larger number of hypernetwork parameters due to the higher dimensionality of dense weight matrices. 
To address this issue, we instead adopt a more parameter-efficient column-wise modulation scheme.
% \roberto{gli output della hypernetwork? questa parte andrebbe meglio legata a quanto si diceva in "Blockwise Parameter Generation", facendo ad esempio un richiamo, ed al più sintetizzare il resto, visto che già si era parlato di questo approccio per rendere il tutto più efficiente} 
In particular, denoting by $\Delta W_\ell$ the weight modulation matrix associated with block $\ell$, we generate its columns individually. 
For each input dimension $j = 1, \dots, d_{\text{in}}$, the $j$-th column is given by
\[
\Delta W_{\ell,:,j} = \left( W_{\text{out}} \,[\vect{x}; \vect{e}_\ell; \vect{e}^{\text{in}}_j] + \vect{b}_{\text{out}} \right)_{1:d_{\text{out}}},
\]
where $\vect{e}^{\text{in}}_j \in \mathbb{R}^{n}$ is a learnable embedding associated with the $j$-th input dimension, 
$W_{\text{out}} \in \mathbb{R}^{d_{\text{out}}^{\max} \times (m + l + n)}$ and $\vect{b}_{\text{out}} \in \mathbb{R}^{d_{\text{out}}^{\max}}$, 
$d_{\text{out}}^{\max}$ denotes the maximum output dimensionality across all fully-connected blocks, and $\Delta W_{\ell,:,j}$ denotes the $j$-th column of the weight modulation matrix $\Delta W_\ell$.

In these experiments, the dimensionality of $\vect{e}_\ell$ is fixed to $16$, while that of $\vect{e}^{\text{in}}$ is fixed to $8$. 
Differently from~\cite{kim2018semi}, we set the learning rate to $1 \times 10^{-4}$. 
We found that using a smaller learning rate for the decoder improves training stability; therefore, we use a learning rate of $5 \times 10^{-5}$ for the decoder while keeping the same value for the remaining parameters.
We additionally employ weight decay with coefficient $5 \times 10^{-3}$ for the OMNIGLOT dataset, and $5 \times 10^{-4}$ for the MNIST dataset, which we found to improve training stability. 
Batch normalization layers in the encoder are kept fixed during IA-VAE training to prevent their running statistics from being affected by instance-specific parameter modulations. 
Since encoder parameters vary across observations, updating these statistics would mix activations from different effective networks, leading to unstable estimates. Keeping them fixed ensures consistent normalization and stabilizes training.
The hypernetwork output layers are initialized with zero-mean Gaussian weights with standard deviation $10^{-3}$ and zero bias, so that the model starts close to the base encoder parameters and progressively learns instance-specific modulations.

These values were determined through preliminary analyses on the considered datasets.

\section{Results}
\label{sec:results}
We report the empirical results for IA-VAE in both the synthetic and image settings introduced in Section~\ref{sec:experiments}. 
We first consider the synthetic dataset, where the true generative process is known and a detailed analysis of posterior inference is possible. 
We then evaluate the proposed approach on standard image benchmarks, where performance is assessed through variational metrics on held-out data.

\subsection{Synthetic Data}
In this subsection, we report the results on the synthetic dataset introduced in Section~\ref{sec:experiments}, where the true generative process is known and the amortization gap can be directly analyzed.
We evaluate IA-VAE along the criteria defined above, including robustness to random initialization, qualitative and quantitative accuracy of the inferred posterior, and parameter efficiency. 

\begin{enumerate}
    \item \textit{Robustness Across Random Initializations.} 
    % \begin{table}[t]
    % \centering
    % \caption{Robustness to random initialization on the synthetic dataset.
    % Each row corresponds to a base VAE trained with a different random initialization. The baseline column reports the ELBO of the base model, with the corresponding KL divergence shown in parentheses.
    % For each base model, 10 IA-VAE instances are trained with different random seeds. The table reports the mean and standard deviation of their ELBO and KL divergence values.
    % Higher ELBO indicates better performance. Comparisons should be made row-wise between the baseline and IA-VAE columns.}
    % \scalebox{.9}{
    %     \begin{tabular}{lccc}
    %     \toprule
    %      & Baseline & IA-VAE (ELBO) & IA-VAE (KL) \\
    %     \midrule
    %     VAE$_1$ 
    %     & $-7.95\;(5.47)$ & $-6.43 \pm 0.08$ & $5.55 \pm 0.01$\\
    %     VAE$_2$ 
    %     & $-7.86\;(5.48)$ & $-6.42 \pm 0.04$ & $5.54 \pm 0.01$\\
    %     VAE$_3$ 
    %     & $-7.91\;(5.48)$ & $-6.43 \pm 0.03$ & $5.53 \pm 0.01$\\
    %     VAE$_4$ 
    %     & $-7.83\;(5.48)$ & $-6.42 \pm 0.04$ & $5.61 \pm 0.02$\\
    %     VAE$_5$ 
    %     & $-7.84\;(5.46)$ & $-6.42 \pm 0.02$ & $5.58 \pm 0.02$\\
    %     VAE$_6$ 
    %     & $-8.60\;(5.50)$ & $-6.45 \pm 0.06$ & $5.54 \pm 0.04$\\
    %     VAE$_7$ 
    %     & $-7.94\;(5.48)$ & $-6.39 \pm 0.05$ & $5.53 \pm 0.02$\\
    %     VAE$_8$ 
    %     & $-7.87\;(5.48)$ & $-6.43 \pm 0.05$ & $5.53 \pm 0.01$\\
    %     VAE$_9$ 
    %     & $-7.95\;(5.48)$ & $-6.39 \pm 0.05$ & $5.54 \pm 0.01$\\
    %     VAE$_{10}$ 
    %     & $-7.84\;(5.49)$ & $-6.35 \pm 0.03$ & $5.56 \pm 0.01$\\
    %     \bottomrule
    %     \end{tabular}
    % }
    % \label{tab:robustness_decoder_oracle}
    % \end{table}
    \begin{table}[t]
    \centering
    \caption{Robustness to random initialization on the synthetic dataset.
    Each row corresponds to a base VAE trained with a different random initialization. 
    For each VAE, 10 IA-VAE instances are trained using different random seeds. 
    The table reports ELBO values (with KL divergence in parentheses), and for IA-VAE we report mean and standard deviation across runs. 
    Comparisons should be made row-wise between VAE and IA-VAE.
    }
    \scalebox{.9}{
        \begin{tabular}{cc}
        \toprule
        VAE & IA-VAE \\
        \midrule
        $-7.95\;(5.47)$ & $\mathbf{-6.43 \pm 0.08 \; (5.55 \pm 0.01)}$\\
        $-7.86\;(5.48)$ & $\mathbf{-6.42 \pm 0.04 \; (5.54 \pm 0.01)}$\\
        $-7.91\;(5.48)$ & $\mathbf{-6.43 \pm 0.03 \; (5.53 \pm 0.01)}$\\
        $-7.83\;(5.48)$ & $\mathbf{-6.42 \pm 0.04 \; (5.61 \pm 0.02)}$\\
        $-7.84\;(5.46)$ & $\mathbf{-6.42 \pm 0.02 \; (5.58 \pm 0.02)}$\\
        $-8.60\;(5.50)$ & $\mathbf{-6.45 \pm 0.06 \; (5.54 \pm 0.04)}$\\
        $-7.94\;(5.48)$ & $\mathbf{-6.39 \pm 0.05 \; (5.53 \pm 0.02)}$\\
        $-7.87\;(5.48)$ & $\mathbf{-6.43 \pm 0.05 \; (5.53 \pm 0.01)}$\\
        $-7.95\;(5.48)$ & $\mathbf{-6.39 \pm 0.05 \; (5.54 \pm 0.01)}$\\
        $-7.84\;(5.49)$ & $\mathbf{-6.35 \pm 0.03 \; (5.56 \pm 0.01)}$\\
        \bottomrule
        \end{tabular}
    }
    \label{tab:robustness_decoder_oracle}
    \end{table}
    
    The results reported in Table~\ref{tab:robustness_decoder_oracle} show that IA-VAE consistently improves the ELBO with respect to the corresponding baseline VAE across all model initializations. While the baseline ELBO varies across different VAE trainings, the IA-VAE models systematically achieve substantially higher values.
    Furthermore, the variability across the 10 IA-VAE runs associated with each base model is very small, as indicated by the low standard deviations reported for each mean. This suggests that the improvements obtained by IA-VAE are robust with respect to the random initialization of the hypernetwork and other stochastic elements of the training procedure.
    The KL divergence values reported for IA-VAE are also highly consistent across seeds and are slightly higher than those of the baseline VAE.
    This suggests that the improvement in ELBO is not due to a weaker regularization effect, but rather to a more effective use of the latent space and a tighter approximation of the posterior.

    \item \textit{Qualitative Posterior Analysis.}
    
    Results are shown in Figure~\ref{fig:posterior_landscape}. We remind the reader that, the plots show the true latent variable $\vect{z}_{true}$ used to generate the observation, the MAP estimate $\vect{z}_{\mathrm{MAP}}$, and the posterior means inferred by the VAE ($\boldsymbol{\mu}_{\mathrm{VAE}}$) and the proposed IA-VAE ($\boldsymbol{\mu}_{\mathrm{IA\!-\!VAE}}$).
    
    \begin{figure}[t]
    \centering
    \begin{subfigure}{0.45\linewidth}
        \centering
        \includegraphics[width=\linewidth]{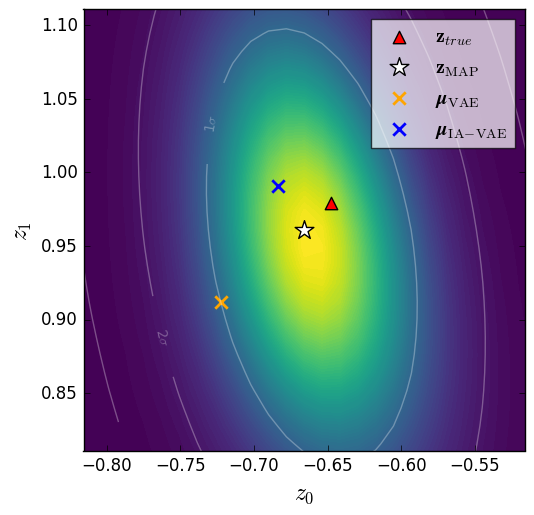}
        % \caption{}
    \end{subfigure}
    \hfill
    \begin{subfigure}{0.45\linewidth}
        \centering
        \includegraphics[width=\linewidth]{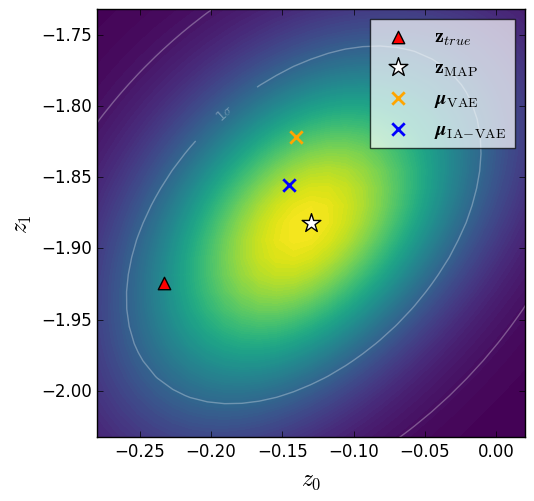}
        % \caption{}
    \end{subfigure}
    \caption{Posterior density in the latent space for two randomly selected observations from the synthetic dataset. For each observation, the posterior is evaluated on a dense grid over the two-dimensional latent space, and the resulting density is shown as a heatmap; brighter regions correspond to higher posterior density. 
    The red triangle denotes the true latent variable $\vect{z}_{true}$ used to generate the observation, the white star denotes the MAP estimate $\vect{z}_{\mathrm{MAP}}$, and the orange and blue crosses denote the posterior means inferred by the VAE ($\boldsymbol{\mu}_{\mathrm{VAE}}$) and the proposed IA-VAE ($\boldsymbol{\mu}_{\mathrm{IA-VAE}}$) model, respectively. 
    The overlaid contours correspond to level sets of a Gaussian approximation of the posterior.
    The MAP estimate $\vect{z}_{\mathrm{MAP}}$ serves as a reference point, as it identifies the most probable latent configuration under the true posterior.}
    \label{fig:posterior_landscape}
    \end{figure}
    
    The qualitative comparison highlights a clear difference between the two inference models. 
    In both examples, the posterior mean inferred by IA-VAE lies closer to the high-density region of the posterior and is better aligned with the MAP estimate. 
    In contrast, the posterior mean produced by the standard VAE tends to be displaced from the posterior mode and falls in regions of lower posterior density.
    This behavior reflects the limited flexibility of amortized inference, where a single set of parameters must generalize across all datapoints, potentially leading to suboptimal approximations for individual observations.
    In contrast, IA-VAE generates instance-specific parameter modulations, enabling the inference network to adapt to the local structure of the posterior.

    \item \textit{Quantitative Posterior Evaluation.}
    \begin{table}[t]
    \centering
    \caption{Quantitative comparison of inference accuracy on the synthetic dataset in the oracle setting. 
    The table reports the ELBO (with the corresponding KL divergence in parentheses), the Mahalanobis distance from the MAP estimate $d_{\mathrm{MAP}}$, and the posterior density ratio $r_{\mathrm{MAP}} = p(\boldsymbol{\mu}|\vect{x}))/p(\vect{z}_{\mathrm{MAP}}|\vect{x}))$, where $\mu$ denotes the inferred posterior mean and $\vect{z}_{\mathrm{MAP}}$ the MAP estimate of the posterior. 
    Arrows indicate the preferred direction for each metric.}
    \begin{tabular}{lcccc}
    \toprule
    Method & $\mathrm{ELBO}\,\uparrow$ & $d_{\mathrm{MAP}}\,\downarrow$ & $r_{\mathrm{MAP}}\,\uparrow$ \\
    \midrule
    VAE & -7.83 (5.48) & 1.44 & 0.53 \\
    IA-VAE & \textbf{-6.36 (5.55)} & \textbf{0.56} & \textbf{0.84} \\
    \bottomrule
    \end{tabular}
    \label{tab:toy_results_oracle}
    \end{table}
    
    From Table~\ref{tab:toy_results_oracle}, we can observe that the IA-VAE achieves a higher ELBO than the standard VAE, while exhibiting a comparable KL divergence, suggesting that the improvement primarily arises from a tighter variational approximation rather than a different level of regularization with respect to the prior.
    The second column reports the Mahalanobis distance from the MAP estimate, $d_{\mathrm{MAP}}$. The IA-VAE significantly reduces this distance compared to the VAE, indicating that the inferred latent representations are located much closer to the high-density region of the posterior density.
    Finally, the third column reports the posterior density ratio $r_{\mathrm{MAP}} = p(\boldsymbol{\mu}|\vect{x})/p(\vect{z}_{\mathrm{MAP}}|\vect{x})$, which measures the posterior density at the inferred mean relative to the density at the MAP. Values close to one indicate that the inferred mean lies in a region of high posterior probability. The IA-VAE attains a substantially higher density ratio than the VAE, confirming that its inferred latent representations concentrate in regions of the latent space with significantly higher posterior probability.

    \item \textit{Parameter Efficiency.}
    In the results shown above, IA-VAE achieves an average ELBO of approximately $-6.40$ using about $68$ parameters in total for the inference mechanism (including both the base encoder and the hypernetwork). By contrast, a standard VAE with a comparable parameter number ($68$ parameters) achieves an ELBO of $-6.84$. A comparable ELBO is only obtained when the size of the standard encoder is increased to approximately $116$ parameters.
    This corresponds to approximately a factor of $1.7$ increase in the number of parameters compared to IA-VAE. Equivalently, IA-VAE attains similar performance while using about $40\%$ fewer parameters. These results suggest that the improvements achieved by IA-VAE cannot be attributed solely to a larger parameter budget. Instead, they indicate that the instance-specific parameter modulations generated by the hypernetwork allow the inference model to use its capacity more efficiently than a standard amortized encoder with globally shared parameters.
    
    \begin{figure}
        \centering
        \includegraphics[width=0.45\linewidth]{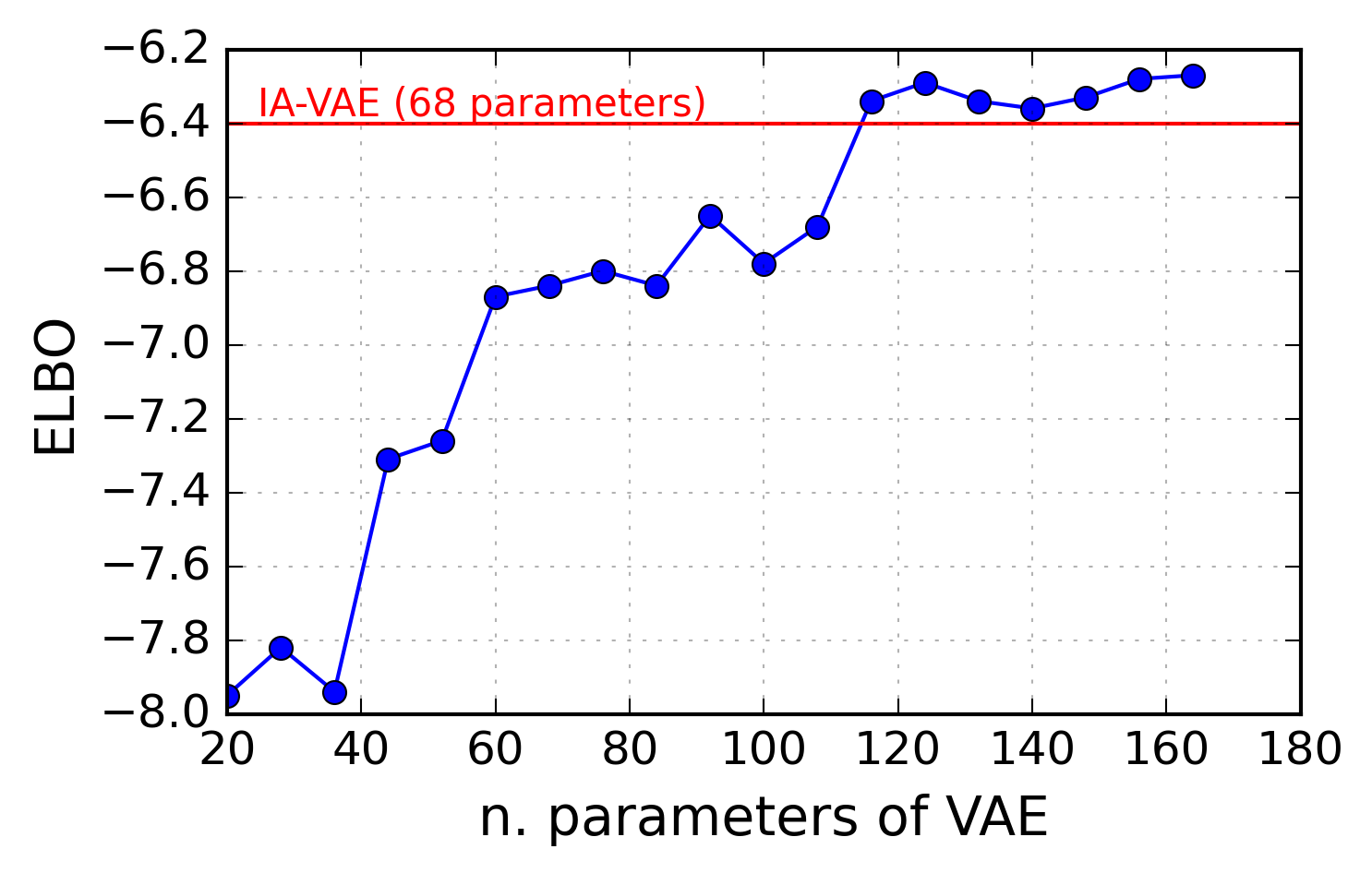}
        \caption{Parameter efficiency comparison between IA-VAE and standard VAEs with increasing encoder capacity in the oracle setting.
        The curve shows the ELBO achieved by baseline VAEs as a function of the number of encoder parameters obtained by progressively increasing the hidden layer size.
        The horizontal red line indicates the ELBO achieved by IA-VAE using a total of 68 parameters for the inference mechanism.}
        \label{fig:placeholder}
    \end{figure}

\end{enumerate}

\subsection{Image Data}
In this subsection, we report the results on the image datasets introduced in Section~\ref{sec:experiments}, which provide a realistic high-dimensional setting where inference quality is assessed through ELBO.

% \begin{table}[h]
% \centering
% \caption{Test ELBO on OMNIGLOT, MNIST, and Fashion MNIST for both VAE and IA-VAE. Results are reported as mean and standard deviation over 10 random seeds. Values in parentheses indicate the corresponding KL divergence term, also reported as mean and standard deviation.}
% \label{tab:results_images_benchmark}
% \vspace{2mm}
% \begin{tabular}{l|cc}
% \toprule
% Dataset & VAE & IA-VAE \\
% \midrule
% OMNIGLOT 
% & $-90.30 \pm 0.03 \;\; (1.01 \pm 0.10)$ 
% & $\mathbf{-89.73 \pm 0.01 \;\; (1.30 \pm 0.05)}$ \\

% MNIST 
% & $-79.42 \pm 0.03 \;\; (4.24 \pm 0.13)$ 
% & $\mathbf{-78.68 \pm 0.01 \;\; (4.31 \pm 0.02)}$ \\

% Fashion MNIST 
% & $-225.87 \pm 0.07 \;\; (4.50 \pm 0.08)$ 
% & $\mathbf{-224.79 \pm 0.01 \;\; (4.91 \pm 0.02)}$ \\

% \bottomrule
% \end{tabular}
% \end{table}

\begin{table}[h]
\centering
\caption{Test ELBO on OMNIGLOT, MNIST, and Fashion MNIST for both VAE and IA-VAE. Results are reported as mean and standard deviation over 10 random seeds. The last column reports the $p$-value of a one-sided paired hypothesis test ($\alpha = 0.05$). Entries marked with $^{*}$ indicate statistically significant improvements of IA-VAE over VAE.}
\label{tab:results_images_benchmark}
\vspace{2mm}
\begin{tabular}{l|cc|c}
\toprule
Dataset & VAE & IA-VAE & $p$-value (ELBO) \\
\midrule
OMNIGLOT 
& $-90.30 \pm 0.03$ 
& $\mathbf{-89.73 \pm 0.01}$ 
& $< 0.01^{*}$ \\

MNIST 
& $-79.42 \pm 0.03$ 
& $\mathbf{-78.68 \pm 0.01}$ 
& $< 0.01^{*}$ \\

Fashion MNIST 
& $-225.87 \pm 0.07$ 
& $\mathbf{-224.79 \pm 0.01}$ 
& $< 0.01^{*}$ \\

\bottomrule
\end{tabular}
\end{table}

Table~\ref{tab:results_images_benchmark} reports the test ELBO on the considered datasets. 
IA-VAE consistently outperforms the standard VAE across all benchmarks, yielding higher ELBO values on OMNIGLOT, MNIST, and Fashion-MNIST.
These results indicate that IA-VAE achieves a better trade-off between reconstruction accuracy and regularization, leading to a tighter variational bound.
From an inference perspective, these results indicate that instance-specific parameter modulations enable a tighter approximation of the true posterior, mitigating the limitations of amortized inference. 
By adapting the encoder parameters to each input, IA-VAE is able to better match the local structure of the posterior, leading to consistent improvements across all datasets.
Reconstruction examples obtained with IA-VAE are shown in Figure~\ref{fig:dataset_reconstruction}.

\begin{figure}
    \centering
    \includegraphics[width=0.9\linewidth]{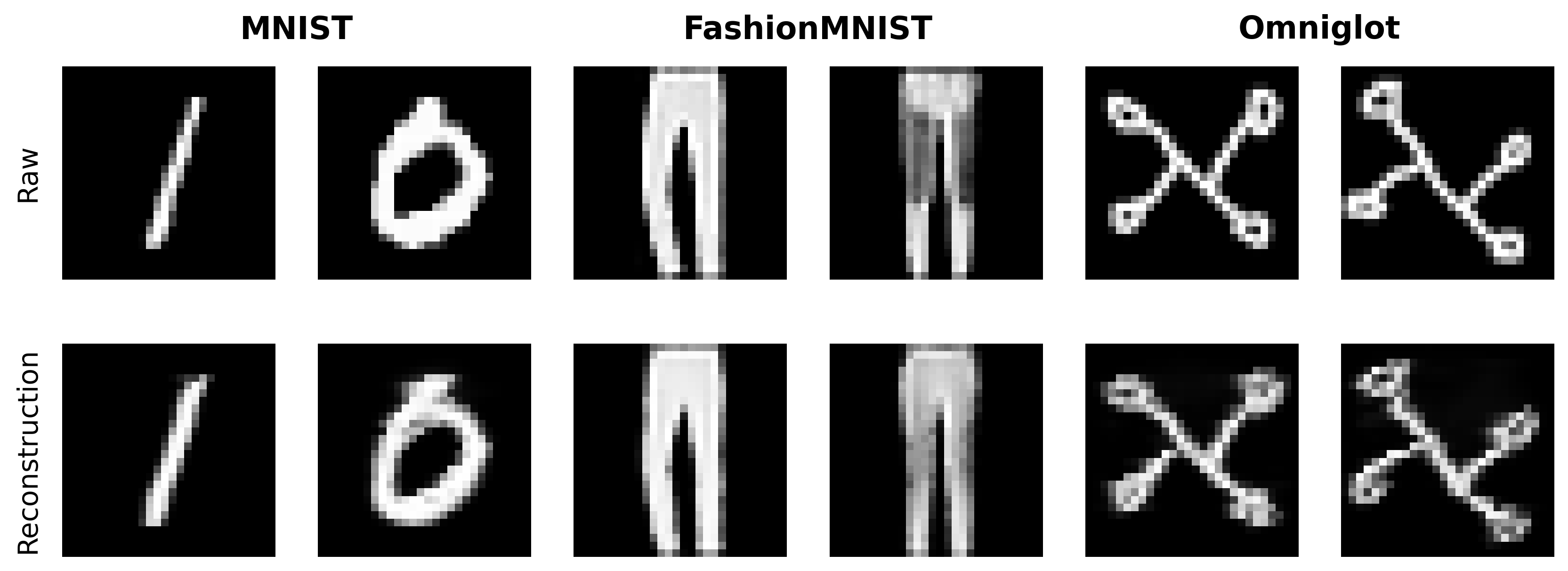}
    \caption{Reconstruction examples on MNIST, Fashion-MNIST, and OMNIGLOT. For each dataset, the top row shows the original inputs and the bottom row the corresponding reconstructions obtained using IA-VAE.}
    \label{fig:dataset_reconstruction}
\end{figure}

Additionally, we performed hypothesis testing to assess whether the observed differences between VAE and IA-VAE are statistically significant.
Because each IA-VAE model is trained starting from the corresponding VAE obtained with the same random seed, comparisons are naturally conducted using paired tests. 
The null and alternative hypotheses are defined as
\[
H_0: \mu_{\mathrm{VAE}} = \mu_{\mathrm{IA\mbox{-}VAE}},
\qquad
H_1: \mu_{\mathrm{VAE}} < \mu_{\mathrm{IA\mbox{-}VAE}},
\]
where the alternative reflects the expectation that IA-VAE attains higher ELBO values than the standard VAE.
For each dataset and for each metric, we first compute the paired differences across seeds. 
We then assess whether these differences are compatible with normality using the Shapiro-Wilk test~\cite{hastie2009elements}. 
If normality is not rejected at significance level $\alpha = 0.05$, we apply a one-sided paired Student's $t$-test~\cite{hastie2009elements}. 
Otherwise, we use a one-sided Wilcoxon signed-rank test~\cite{hastie2009elements}, which does not rely on the normality assumption.

Results are reported in Table~\ref{tab:results_images_benchmark}. 
The improvement in ELBO achieved by IA-VAE is statistically significant across all considered datasets according to one-sided paired tests, providing consistent evidence that the proposed method yields a better variational fit to the data compared to the standard VAE.
Importantly, the gains are observed uniformly across datasets and exhibit low variability across random seeds, suggesting that the benefits of instance-adaptive inference are both robust and reproducible in high-dimensional settings.

\section{Conclusions}
\label{sec:conclusions}

In this work, we introduced IA-VAE, an instance-adaptive parametrization of amortized variational inference within the VAE framework, in which a hypernetwork generates input-dependent modulations of a shared inference model. This approach relaxes the structural constraint imposed by globally shared encoder parameters, while preserving the computational efficiency of amortized inference.

From a theoretical perspective, we showed that the variational family induced by IA-VAE contains that of standard amortized inference under mild conditions, ensuring that the proposed method preserves all solutions achievable by the base encoder and cannot yield a worse optimal ELBO.

Empirically, we evaluated IA-VAE on both synthetic and real-world datasets. In the synthetic setting, where the true posterior is known, IA-VAE produces more accurate posterior approximations and reduces the amortization gap. On standard image benchmarks, it consistently improves held-out ELBO with statistically significant gains across multiple runs. These improvements are obtained without introducing iterative refinement or additional optimization, maintaining the scalability advantages of amortized inference.

Overall, these results support the view that instance-specific parameter modulation provides a principled and effective mechanism to mitigate the limitations of standard amortized inference. By enabling input-dependent adaptation of the inference model, IA-VAE yields more accurate and robust variational approximations.

The proposed framework also opens several directions for future work. In particular, input-dependent parameter generation naturally connects to continual learning settings~\cite{wang2024comprehensive,von2019continual}, where models must efficiently adjust to evolving data distributions. 
Furthermore, investigating the impact of instance-adaptive inference on representation learning remains a compelling direction; specifically, we aim to explore how IA-VAE can enhance feature disentanglement \cite{locatello2020weakly,togo2024concvae}.
By enabling the encoder to better capture factorized and semantically meaningful representations, this approach could significantly improve interpretability within the explainable artificial intelligence context \cite{Apicella20221,apicella2023strategies}, as it explicitly exposes how the inference process adapts and isolates key features across different observations.
% Exploring these aspects may further clarify the role of instance-adaptive inference in enhancing both interpretability and representation quality in deep generative models.

%
\section*{Contribute statement}
\textbf{Conceptualization:} A. Pollastro, R. Prevete;
\textbf{Methodology:} A. Pollastro;
\textbf{Software and Experiments:} A. Pollastro;
\textbf{Investigation:} A. Pollastro;
\textbf{Formal Analysis:} A. Pollastro, F. Isgrò;
\textbf{Writing:} A. Pollastro, A. Apicella, R. Prevete;
\textbf{Supervision:} R. Prevete, F. Isgrò.

\section*{Acknowledgements}
\label{sec:Acknowledgements}
This work was funded by the PNRR MUR project PE0000013-FAIR (CUP: E63C25000630006).

\bibliographystyle{unsrt}
\bibliography{bibliography}
\end{document}